\documentclass[twoside,leqno,twocolumn]{article}  
\usepackage{ltexpprt}
\usepackage{amsmath}
\usepackage{graphicx}
\usepackage{algorithm}
\usepackage{algorithmic}
\usepackage{subcaption}

\DeclareMathOperator{\sign}{sign}
\DeclareMathOperator{\conf}{conf}

\newtheorem{definition}{Definition}

\title{A Confidence-Based Approach for Balancing Fairness and Accuracy}
\author{Benjamin Fish\and Jeremy Kun\and \'Ad\'am D. Lelkes\thanks{University
of Illinois at Chicago, Department of Mathematics, Statistics, and Computer
Science} \thanks{\{bfish3, jkun2, alelke2\}@uic.edu}}

\begin{document}

\maketitle

\begin{abstract} 

We study three classical machine learning algorithms in the context of
algorithmic fairness: adaptive boosting, support vector machines, and logistic
regression.  Our goal is to maintain the high accuracy of these learning
algorithms while reducing the degree to which they discriminate against
individuals because of their membership in a protected group. 

Our first contribution is a method for achieving fairness by shifting the decision
boundary for the protected group.  The method is based on the theory of margins
for boosting.  Our method performs comparably to or outperforms previous algorithms
in the fairness literature in terms of accuracy and low discrimination,
while simultaneously allowing for a fast and transparent quantification
of the trade-off between bias and error. 

Our second contribution addresses the shortcomings of the bias-error trade-off
studied in most of the algorithmic fairness literature. We demonstrate that
even hopelessly naive modifications of a biased algorithm, which cannot be
reasonably said to be fair, can still achieve low bias and high accuracy.  To
help to distinguish between these naive algorithms and more sensible algorithms
we propose a new measure of fairness, called \emph{resilience to random bias}
(RRB). We demonstrate that RRB distinguishes well between our naive and
sensible fairness algorithms.  RRB together with bias and accuracy provides a
more complete picture of the fairness of an algorithm. 

\end{abstract}

\section{Background and Motivation} \label{sec:background}

\subsection{Motivation}

Machine learning algorithms assume an increasingly large role in making
decisions across many different areas of industry, finance, and government,
from facial recognition and social network analysis to self-driving cars to
data-based approaches in commerce, education, and policing. The decisions made
by algorithms in these domains directly affect individual people, and not
always for the better.  Consequently, there has been a growing concern that
machine learning algorithms, which are often poorly understood by those that
use them, make discriminatory decisions.

If the data used for training the algorithm is biased, a machine learning
algorithm will learn the bias and perpetuate discriminatory decisions against
groups that are protected by law, even in the absence of ``discriminatory
intent'' by the designers. A typical example is an algorithm serving predatory
ads to protected groups. Such issues resulted in a 2014 report from the US
Executive Office~\cite{PodestaPMHZ14} which voiced concerns about
discrimination in machine learning. The primary question we study in this paper
is
\begin{center}
How can we maintain high accuracy of a learning algorithm while reducing
discriminatory biases?
\end{center}
In this paper we will focus on the issue of biased training data, which is one
of the several possible causes of discriminatory outcomes in machine learning.
In this setting, we have a protected attribute (e.g. race or gender) which we
assert should be independent from the target attribute.  For example, if the
goal is to decide creditworthiness for loans and the protected attribute is
gender, a classifier's prediction should not correlate with an applicant's
gender. We say that the classifier achieves \emph{statistical parity} if the
protected subgroup is as likely as the broader population to have a given
label.

Of course, there might be situations where the target label depends on
legitimate factors that correlate with the protected attribute. For example,
if the protected attribute is gender and the target label is income, some argue
that lower salaries for women can be partly explained by the fact that on
average, men work longer hours than women. In this paper we assume that this
is not the case. The issue of ``explainable discrimination'' in machine
learning was studied in \cite{KamiranZC13}.

In our setting, since we only have biased data, we cannot evaluate our
classifiers against an unbiased ground truth. In particular only a biased
classifier could achieve perfect accuracy; to achieve statistical parity in
general one must be willing to reduce accuracy. Hence the natural goal is to
find a classifier that achieves statistical parity while minimizing error, or
more generally to study the trade-off between bias and accuracy so as to make
favorable trade-offs.
\subsection{Contributions}
Our first contribution in this paper is a method for optimizing this trade-off
which we call the \emph{Shifted Decision Boundary} (SDB). SDB is a generic
method based on the theory of margins~\cite{CortesV95,SchapireFBL98}, and it
can be combined with any learning algorithm that produces a measure of
confidence in its prediction  (Section~\ref{sec:sdb}). In particular we combine SDB with boosting,
support vector machines, and logistic regression, and it performs comparably to
or outperforms previous algorithms in the fair learning literature.  See Section~\ref{sec:experiments} for its empirical evaluation.
We also give a
theorem based on the analysis in~\cite{SchapireFBL98} bounding the loss of
accuracy for SDB under weighted majority schemes (Section~\ref{sec:sdbtheory}). 
SDB makes the assumptions on the bias explicit and
transparent, so that the trade-off can be understood without a detailed
understanding of the learning algorithm itself. 

Unfortunately, studying the bias-error trade-off is an incomplete picture of
the fairness of an algorithm. The shortcomings were discussed
in~\cite{DworkHPR12}, e.g., in terms of how an adversary could achieve
statistical parity while still targeting the protected group unfairly. We
demonstrate these shortcomings in action even in the absence of adversarial
manipulation. Among other methods, we show that modifying a classifier by
randomly flipping certain output labels with a certain probability already
outperforms much of the prior fairness literature in both accuracy and bias.
Such a naive algorithm is obviously unfair because the relabeling is
independent of the classification task. Our second contribution is a measure of
fairness that addresses this shortcoming, which we call \emph{resilience to
random bias}. We define it in Section~\ref{sec:rrb} and demonstrate that it
distinguishes well between our naive baseline algorithms and SDB.

\subsection{Existing notions of fairness} 

The study of fairness in machine learning is young, but there has been a lot of
disparate work studying notions of what it means for data to be fair. Finding 
the ``right'' definition of fairness is a major challenge; see the extensive
survey of~\cite{RomeiR14} for a detailed discussion. Two prominent definitions
of fairness that have emerged are \emph{statistical parity} and
\emph{$k$-nearest-neighbor consistency.} We review them briefly now.

\emph{Statistical parity:} Let $D$ be a distribution over a set of labeled
examples $X$ with labels $l : X \to \{-1, 1\}$ and a protected subset $S \subset
X$. The \emph{bias} of $l$ with respect to $D$ is defined as the difference in
probability of an example in $S$ having label 1 and the probability of an
example in $S^C$ having label 1, i.e.  $$ B(D, S) = \Pr_{x \sim D|_{S^C}}[l(x)
= 1] - \Pr_{x \sim D|_{S}}[l(x) = 1].  $$ The bias of a hypothesis $h$ is the
same quantity with $h(x)$ replacing $l(x)$.  If a hypothesis has low bias in
absolute value we say it achieves \emph{statistical parity.} Note that $S$
represents the group we wish to protect from discrimination, and the bias
represents the degree to which they have been discriminated against.  The sign
of bias indicates whether $S$ or $S^C$ is discriminated against.  A similar
statistical measure called \emph{disparate impact} was introduced and studied
by Friedler et al.~\cite{FriedlerSV14} based on the ``$80\%$ rule'' used in
United States hiring law.

Dwork~et~al.~\cite{DworkHPR12} point out that statistical parity is only a
measure of population-wide fairness. They provide a laundry list
of ways one could achieve statistical parity while still exhibiting serious and
unlawful discrimination. In particular, one can achieve statistical parity by
flipping the labels of a certain number of arbitrarily chosen members of the
disadvantaged group, regardless of the relation between the individuals and the
classification task. In our experiments we show this already outperforms some
of the leading algorithms in the fairness literature.

Despite this, it is important to study the ability for learning algorithms to
achieve statistical parity.  For example, it might be
reasonable to flip the labels of the ``most qualified'' individuals of the
disadvantaged group who are classified negatively. Some previous approaches
assume the existence of a ranking or metric on individuals, or try to learn
this ranking from data~\cite{KamiranC09,DworkHPR12}. By contrast, our SDB
achieves statistical parity without the need for such a ranking.

\emph{$kNN$-consistency:} The second notion, due to~\cite{DworkHPR12}, calls a
classifier ``individually fair'' if it classifies similar individuals
similarly. They use $k$-nearest-neighbor to measure the consistency of labels
of similar individuals. Note that ``closeness'' is defined with respect to a
metric chosen as part of the data cleaning and feature selection process. By
contrast SDB does not require a metric on individuals.

\subsection{Previous work on fair algorithms} Learning algorithms studied
previously in the context of fairness include naive Bayes~\cite{CaldersV10},
decision trees~\cite{KamiranCP10}, and logistic
regression~\cite{KamishimaAAS12}.  To the best of our knowledge we are the
first to study boosting and SVM in this context, and our confidence-based
analysis is new for both these and logistic regression. 

The two main approaches in the literature are massaging and regularization.
Massaging means changing the biased dataset before training to remove the bias
in the hope that the learning algorithm trained on the now unbiased data will
be fair.  Massaging is done in the previous literature based on a ranking
learned from the biased data~\cite{KamiranC09}. The regularization
approach consists of adding a regularizer to an optimization objective which
penalizes the classifier for discrimination~\cite{KamashimaAS11}. While SDB can
be thought of as a post-processing regularization, it does so in a way that
makes the trade-off between bias and accuracy transparent and easily
controlled. 

There are two other notable approaches in the fairness literature. The first,
introduced in \cite{DworkHPR12}, is a framework for maximizing the utility of a
classification with the constraint that similar people be treated similarly.
One shortcoming of this approach is that it relies on a metric on the data that
tells us the similarity of individuals with respect to the classification task.
Moreover, the work in~\cite{DworkHPR12} suggests that learning a suitably fair
similarity metric from the data is as hard as the original problem of finding a
fair classifier. Our SDB method does not require such a metric.

The ``Learning Fair Representations'' method of Zemel et al.~\cite{ZemelWSPD13}
formulates the problem of fairness in terms of intermediate representations:
the goal is to find a representation of the data which preserves as much
information as possible from the original data while simultaneously obfuscating
membership in the protected class.  Given that in this paper we seek to make
explicit the trade-off between bias and accuracy, we will not be able to hide
membership in the protected class as Zemel et al.~seeks to do. Rather, we align
with the central thesis of~\cite{DworkHPR12}, that knowing the protected
feature is useful to promote fairness. 

\subsection{Margins}

The theory of margins has provided a deep, foundational explanation for the
generalization properties of algorithms such as AdaBoost and soft-margin
SVMs~\cite{CortesV95,SchapireFBL98}. A hypothesis $f: X \to [-1,1]$ induces a
\emph{margin} for a labeled example $\textup{margin}_f(x,y) = y\cdot f(x)$,
where $x\in X$ is a data point and $y\in\{-1,1\}$ is the correct label for $x$.
The sign of the margin is positive if and only if $f$ correctly labels $x$, and
the magnitude indicates how confident $f$ is in its prediction.

As an example of the power of margins, we quote a celebrated theorem on the
generalization accuracy of weighted majority voting schemes in PAC-learning.
Here a weighted majority vote is a function $f(x) = \sum_{i=1}^N \alpha_i
h_i(x)$ for some hypotheses $h_i \in H$ and $\alpha_i \geq 0, \sum_i \alpha_i =
1$.

\begin{theorem}[Schapire et al.~\cite{SchapireFBL98}]\label{thm:margin-generalization}
Let $D$ be a distribution over $X \times \{ -1,1\}$ and $S$ be a sample of $m$
examples chosen i.i.d. at random according to $D$. Let $H$ be a set of
hypotheses of VC-dimension $d$. Then for any $\delta > 0$, with probability at
least $1-\delta$ \emph{every} weighted majority voting scheme satisfies the
following for every $\theta > 0$:
$$
\begin{aligned}\Pr_D[yf(x) \leq 0] &\leq \Pr_S[yf(x) \leq \theta] + \\ & O \left (
\frac{1}{\sqrt{m}} \left ( \frac{d \log^2(m/d)}{\theta^2} + \log (1/\delta) \right )^{1/2} \right )
\end{aligned}
$$
\end{theorem}

In other words, the generalization error is bounded by the probability of a
small margin \emph{on the sample}. One can go on to show
AdaBoost~\cite{SchapireF12}, a popular algorithm that produces a weighted
voting scheme, performs well in this respect. Recall that the output of
AdaBoost is a hypothesis which outputs the sign of a weighted majority vote
$\sum_i \alpha_i, h_i(x)$. Rather than measure the margin we measure the
\emph{signed confidence} of the boosting hypothesis on an unlabeled example $x$
as $$ \conf(\mathbf x) = \frac{\sum_{i=1}^T \alpha_i h_i(\mathbf
x)}{\sum_{i=1}^T \alpha_i}.$$ The magnitude of the confidence measures the
agreement of the voters in their classification of an example.

The theoretical work on margins for boosting suggests that examples with small
confidence are more likely to have incorrect labels than examples with large
confidence. For example, we display in Figure~\ref{fig:boosting-margins} the
signed confidence values for all examples and incorrectly predicted examples respectively.
The incorrect examples have confidence centered around zero. One
can leverage this for fairness by flipping negative labels of members of the
protected class with a small confidence value. This is a rough sketch of the SDB
method. The empirical results of SDB suggest that SDB achieves statistical
parity with relatively little loss in accuracy. Indeed, we state a similar guarantee
to Theorem~\ref{thm:margin-generalization} in Section~\ref{sec:sdbtheory} that
solidifies this intuition.

\begin{figure}[t]
\centering
\includegraphics[width=\columnwidth]{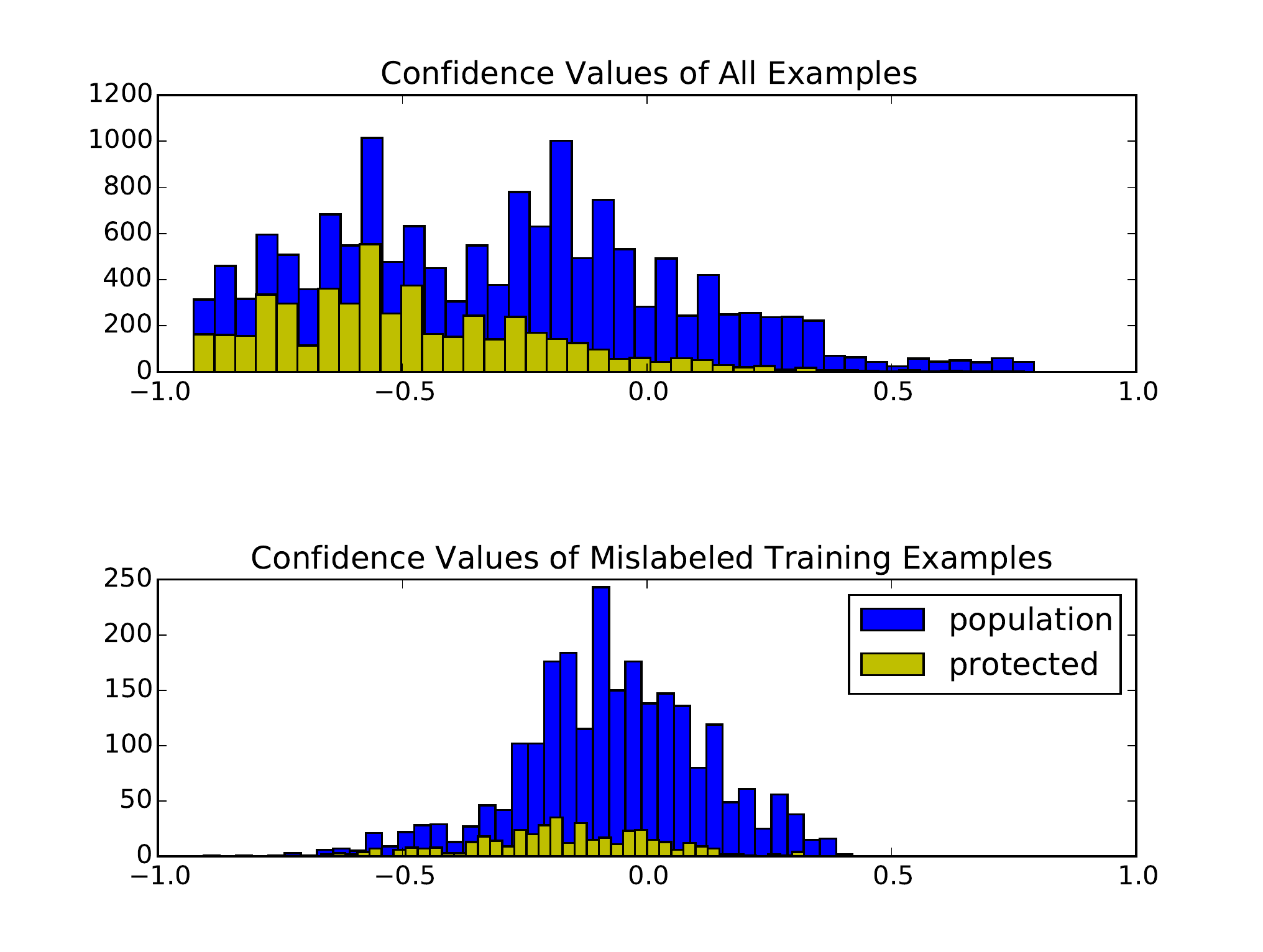}
\caption{Histogram of boosting confidences for the Census data set. The top
histogram shows the distribution of confidence values for the entire dataset,
and the bottom shows the confidence for only mislabeled examples. The vast
majority of women are classified as $-1$, and the incorrect classifications are
closer to the decision boundary.}
\label{fig:boosting-margins}
\end{figure}

The idea of a signed confidence generalizes nicely to other machine
learning algorithms. We study support vector machines (SVM) which have a
natural geometric notion of margin, and logistic regression which outputs a
confidence in its prediction. For background on SVM, logistic regression, and
AdaBoost, see~\cite{ShalevShwartzBD14}.

\subsection{Interpretations of signed confidence}
Here we state how signed confidence is defined for each of the learning
methods.

\subsubsection{AdaBoost}
Boosting algorithms work by combining \emph{base hypotheses}, ``rules of
thumb'' that have a fixed edge over random guessing, into highly accurate
predictors.  In each round, a boosting algorithm finds the base hypothesis that
achieves the smallest weighted error on the sample. It then increases the
weights of the incorrectly classified examples, thus forcing the base learner
to improve the classification of difficult examples. In this paper we study
AdaBoost, a ubiquitous boosting algorithm. For more on boosting, we refer the
reader to~\cite{SchapireF12}. 

Let $H$ be a set of base classifiers, and let $(\alpha_t, h_t)_{t=1}^T$ be the
weights and hypotheses output by AdaBoost after $T$ rounds.  The signed
confidence of the hypothesis is $\conf(\mathbf x) = \frac{\sum_{i=1}^T \alpha_i
h_i(\mathbf x)}{\sum_{i=1}^T \alpha_i}.$ In all of our experiments we boost
decision stumps for $T=20$ rounds.

\subsubsection{SVM} The soft-margin SVM of Vapnik~\cite{CortesV95} outputs a
maximum margin hyperplane $\mathbf{w}$ in a high-dimensional space implicitly
defined by a kernel $K$, and $\mathbf{w}$ can be expressed implicitly as a
linear combination of the input vectors, say $\mathbf{w'}$.  We define the
confidence as the distance of a point from the separating hyperplane, i.e.
$\conf(\mathbf x) = K (\mathbf{w'}, \mathbf x)$. For the Census Income and
Singles datasets we use the standard Gaussian kernel, and for the German
dataset we use a linear kernel (the datasets are described in
Section~\ref{sec:experiments}). 

\subsubsection{Logistic regression}
The classifier output by logistic regression has the form $$h(\mathbf
x)=\sign(\phi(\langle \mathbf w, \mathbf x\rangle) - 1/2)$$ where
$\phi(z)=\frac{1}{1+e^{-z}}$ is the logistic function, and the vector $\mathbf
w$ is found by empirical risk minimization (ERM) with the standard logistic
loss $\ell(\mathbf w, (\mathbf x,y))=\log(1+e^{-y\langle \mathbf w, \mathbf
x\rangle})$ and $L_2$ regularization. Here we define the confidence of logistic
regression simply as the value that the classifier takes before rounding:
$\conf(\mathbf x) = \phi(\langle \mathbf w, \mathbf x\rangle).$

\section{Methods and Technical Solutions} \label{sec:methods}

\subsection{Shifted decision boundary}\label{sec:sdb}

In this section we define our methods. In what follows $X$ is a labeled
dataset, $l(x)$ are the given labels, and $S \subset X$ is the protected group.
We further assume that members of $S$ are less likely than $S^C$ to have label
1.
First we describe our proposed method, called \emph{shifted decision
boundary} (SDB), and then we describe three techniques we use for baseline
comparisons (in addition to comparing to previous literature).

Let $\conf:X \to [-1,1]$ be a function corresponding to a classifier $h(x)=\sign(\conf(x))$,
and define the \emph{decision boundary shift of $\lambda$ for
$S$} as the classifier $h_\lambda: X \to \{-1,1\}$, defined as
$$
h_\lambda(x) = \begin{cases}
1    & \textup{ if } x \in S, \conf(x) \geq -\lambda \\ 
\sign(\conf(x)) & \textup{ otherwise.}
\end{cases}
$$
The SDB algorithm accepts as input confidences $\conf$ and finds the minimal error
decision boundary shift for $S$ that achieves statistical parity. That is,
given $\conf$ and $\varepsilon > 0$, it produces a value $\lambda$ such that $h_\lambda$
has minimal error subject to achieving statistical parity up to bias
$\varepsilon$. 

\subsection{Naive baseline algorithms}

We define two naive baseline methods which are intended to be both baseline
comparisons for our SDB algorithm and illustrations of the shortcomings of the
bias-error trade-off.

Similarly to SDB, the \emph{random relabeling} (RR) algorithm modifies a given
hypothesis $h$ by flipping labels. In particular, RR computes the probability
$p$ for which, if members of $S$ with label $-1$ under $h$ are flipped by $h'$
to $+1$ randomly and independently with probability $p$, the bias of $h'$ is
zero in expectation. The classifier $h'$ is then defined as the randomized
classifier that flips members of $S$ with label $-1$ with probability $p$ and
otherwise is the same as $h$.

Next, we define \emph{random massaging} (RM).  Massaging strategies, introduced
by~\cite{KamiranC09}, involve eliminating the bias of the training data by
modifying the labels of data points, and then training a classifier on this
data in the hope that the statistical parity of the training data will
generalize to the test set as well.  In our experiment, we massage the data
randomly; i.e.~we flip the labels of $S$ from $-1$ to $+1$ independently at
random with the probability needed to achieve statistical parity in
expectation, as in RR.

As we have already noted, these two baseline methods perform comparably to much
of the previous literature in both bias and error. This illustrates that the
semantics of \emph{why} an algorithm achieves statistical parity is crucial
part of its evaluation. As such, these two baselines can be useful for any
analysis that measures bias and accuracy. Moreover, they can
be used to determine the suitability of a new proposed measure of fairness.

\subsection{Fair weak learning}

Finally, we include a method which is based on a natural idea but is
empirically suboptimal to SDB. Recall that boosting works by combining weak
learners into a ``strong'' classifier.  It is natural to ask whether boosting
keeps the fairness properties of the weak learners. Weak learners used in
practice, such as decision stumps, have very low complexity, therefore it is
easy to impose fairness constraints on them. In our \emph{fair weak learning}
(FWL) baseline we replace a standard boosting weak learner with one which tries
to minimize a linear combination of error and bias and run the resulting
boosting algorithm unchanged. The weak learner we use computes the decision
stump which minimizes the sum of label error and bias of its induced
hypothesis. 

\subsection{Theoretical properties of SDB}\label{sec:sdbtheory}

Because the SDB method only flips the labels of examples with small signed
confidence, margin theory implies that it will not increase the error too much.
We formalize this precisely below. This theorem, a direct corollary of
Theorem~\ref{thm:margin-generalization}, provides strong theoretical
justification for our SDB method. To the best of our knowledge, SDB is the
first empirically tested method for fair learning that has any specific
guarantees for its accuracy.

Informally, the theorem says that when a majority voting scheme is
post-processed by the SDB technique, the resulting hypothesis maintains the
generalization accuracy bounds in terms of the margin on the sample when the
shift is small ($\lambda \leq \theta$). But as the shift grows, the error bound
increases proportionally to the fraction of the protected population that has
large enough negative margins (i.e., in $[-\lambda, -\theta]$).

\begin{theorem} 
Let $X$ be finite and $D,S,m,H$, and $d$ be as in
Theorem~\ref{thm:margin-generalization}. Let $T \subset S$ be the subset of the
sample in the protected class. Let $\delta > 0$. Let $\textup{err}(m)$ be the
tail error function from Theorem~\ref{thm:margin-generalization}. For any $A
\subset X$ let $A_{\lambda, \theta} = \{ a \in A : -\lambda \leq
\textup{conf}(a) \leq -\theta \}$. 

Then with probability at least $1-\delta$,
every function $h_\lambda$ post-processed by SDB with weighted majority vote $\conf(x)$ and shift
$\lambda > 0$ satisfies the following for every $\theta > 0$:
$$
\begin{aligned}\Pr_D[yh_\lambda(x) \leq 0] &\leq 
\Pr_{T_{\lambda, \theta}}[y\cdot\conf(x) \geq -\theta] \Pr_{S}[x \in T_{\lambda,
\theta}] \\ & + 
\Pr_{S - T_{\lambda, \theta}}[y\cdot\conf(x) \leq \theta] \Pr_{S}[x \not \in T_{\lambda,
\theta}] \\ & + 
\max ( \textup{err}(|T_{\lambda, \theta}|), \textup{err}(|T^C_{\lambda,
\theta}|))
\end{aligned}
$$  
\end{theorem}
\begin{proof}
The bound follows by conditioning on the event that $h_\lambda$ flips the
label, noticing $-\conf(x)$ is also a majority function, and applying
Theorem~\ref{thm:margin-generalization} twice.  
\end{proof}

\subsection{Resilience to random bias}\label{sec:rrb}
One of the biggest challenges for designers of fair learning algorithms is
the lack of good measures of fairness. The most popular measures are
statistical measures of bias such as statistical parity. As Dwork et
al.~\cite{DworkHPR12} have pointed out, statistical parity fails to capture all
important aspects of fairness. In particular, it is easy to achieve statistical
parity simply by flipping the labels of an arbitrary set of individuals in the
protected class. A real-world example would be giving a raise to a random group
of women to eliminate the gender disparity in wages.  The root cause of this
problem is that one does not have access to reliable (unbiased) ground truth
labels. We propose to compensate for this by evaluating algorithms on synthetic
bias. In doing this we make transparent the \emph{kind} of bias a claimed
``fair'' algorithm protects against, and we can accurately measure its
resilience to said bias.

We introduce a new notion of fairness called \emph{resilience to random bias}
(RRB). Informally we introduce a new, random feature which has no correlation
with the target attribute, and then we introduce bias against individuals which
have a certain value for this new feature.  We call an algorithm fair if it can
recover the original, unbiased labels. For RRB in particular, the synthetic
bias is i.i.d. random against the protected group.  

Certainly, in practice, bias may not be of this form and we do not pretend that this notion captures all forms of bias.
Rather, this notion seeks to model a comparatively mild form of bias --
if an algorithm cannot recover from this type of random bias against a protected class then
there is little reason to think it can handle other types of bias.
In other words, we propose this as a minimally necessary condition but not necessarily a sufficient condition for individual fairness.
Relating our RRB measure more formally to other notions of individual fairness is left for future work.

We formally define RRB as follows. Let $X$ be a set of examples and $D$ be a
distribution over examples, with $l:X \to \{-1,1\}$ a target labeling function.
We first define a randomized process mapping $(X,D,l) \to (\tilde X, \tilde D,
\tilde l)$. Let $\tilde X = X \times \{ -1,1 \}$ and $\tilde D$ be the
distribution on $\tilde X$ which is independently $D$ on the $X$ coordinate and
uniform on the $\{-1,1\}$ coordinate. Denote by $\tilde X_0 = \{ (x,b) \in
\tilde X \mid b=0 \}$ and call this the \emph{protected set}. Finally, $\tilde
l(x,b)$ is \emph{fixed} to either $l(x)$ or $-l(x)$ independently at random
for each $(x,b) \in \tilde X$ according to the following: $$ \Pr[\tilde l(x,b)
= l(x)] = \begin{cases} 1      & \textup{ if } b=1 \textup{ or } l(x) = -1 \\
1-\eta & \textup{ if } b=0 \textup{ and } l(x) = 1 \\ \end{cases}.  $$

In other words, the positive labels of a randomly chosen protected subgroup are
flipped to negative independently at random with probability $\eta$. We
emphasize that the process mapping $l \mapsto \tilde l$ is randomized, but the
map $\tilde l(x,b)$ itself is fixed and deterministic. So an algorithm which
queries labels from $\tilde l$ is given consistent answers. Now we define the
resilience to random bias as follows:

\begin{definition} \label{def:rrb}
Let $(X,D,l), (\tilde X, \tilde D, \tilde l)$ be as above. Let $h = A(\tilde D,
\tilde l)$ be the output classifier of a learning algorithm $A$ when given
biased data as input. The \emph{resilience to random bias} (RRB) of $A$ with
respect to $(X,D,l)$ and discrimination rate $0 \leq \eta < 1/2$, denoted
$\textup{RRB}_{\eta}(A)$, is
$$
\textup{RRB}_{\eta}(A) = \Pr_{\tilde D} [h(x,b) = l(x) \mid b = 0, l(x) = 1]
$$
\end{definition}

Similarly to calculating statistical parity, RRB is estimated on a fixed
dataset by simulating the process described above and outputing an empirical
average. 

\section{Empirical Evaluation} \label{sec:experiments}

We measure our methods on label error, statistical parity, and RRB with $\eta =
0.2$. In all of our experiments we split the datasets randomly into training,
test, and model-selection subsets, and we output the average of 10
experiments.\footnote{The code is available for reproducibility
at~\url{https://github.com/j2kun/fkl-SDM16}}

\subsection{Datasets}

The Census Income dataset~\cite{Lichman13}, extracted from the 1994 Census
database, contains demographic information about $48842$ American adults.  The
prediction task is to determine whether a person earns over \$50K a year.  The
dataset contains $16,192$ females ($33\%$) and $32,650$ males. Note $30.38\%$
of men and $10.93\%$ of women reported earnings of more than \$50K, therefore
the bias of the dataset is $19.45\%$.  

The German credit dataset~\cite{Lichman13} contains financial information about
1000 individuals who are classified into groups of good and bad credit risk.
The ``good'' credit group contains $699$ individuals. Following the work
of~\cite{KamiranC09}, we consider age as the protected attribute with a cut-off
at $25$. Only $59\%$ of the younger people are considered good credit risk,
whereas of the $25$ or older group, $72\%$ are creditworthy, making the bias
$13\%$.

In the Singles dataset, extracted from the marketing dataset
of~\cite{HastieTF09} by taking all respondents who identified as ``single,''
the goal is to predict whether annual income of a household is greater than
\$25K from 13 other demographic attributes.  The protected attribute is gender.
The dataset contains $3,653$ data points, $1,756$ ($48\%$) of which belong to
the protected group. $34\%$ of the dataset has a positive label. The bias is
$9.8\%$.

\subsection{Results and analysis}\label{sec:results}

\begin{figure*}[t]
\centering
\begin{subfigure}{0.65\columnwidth}
\includegraphics[width=\columnwidth]{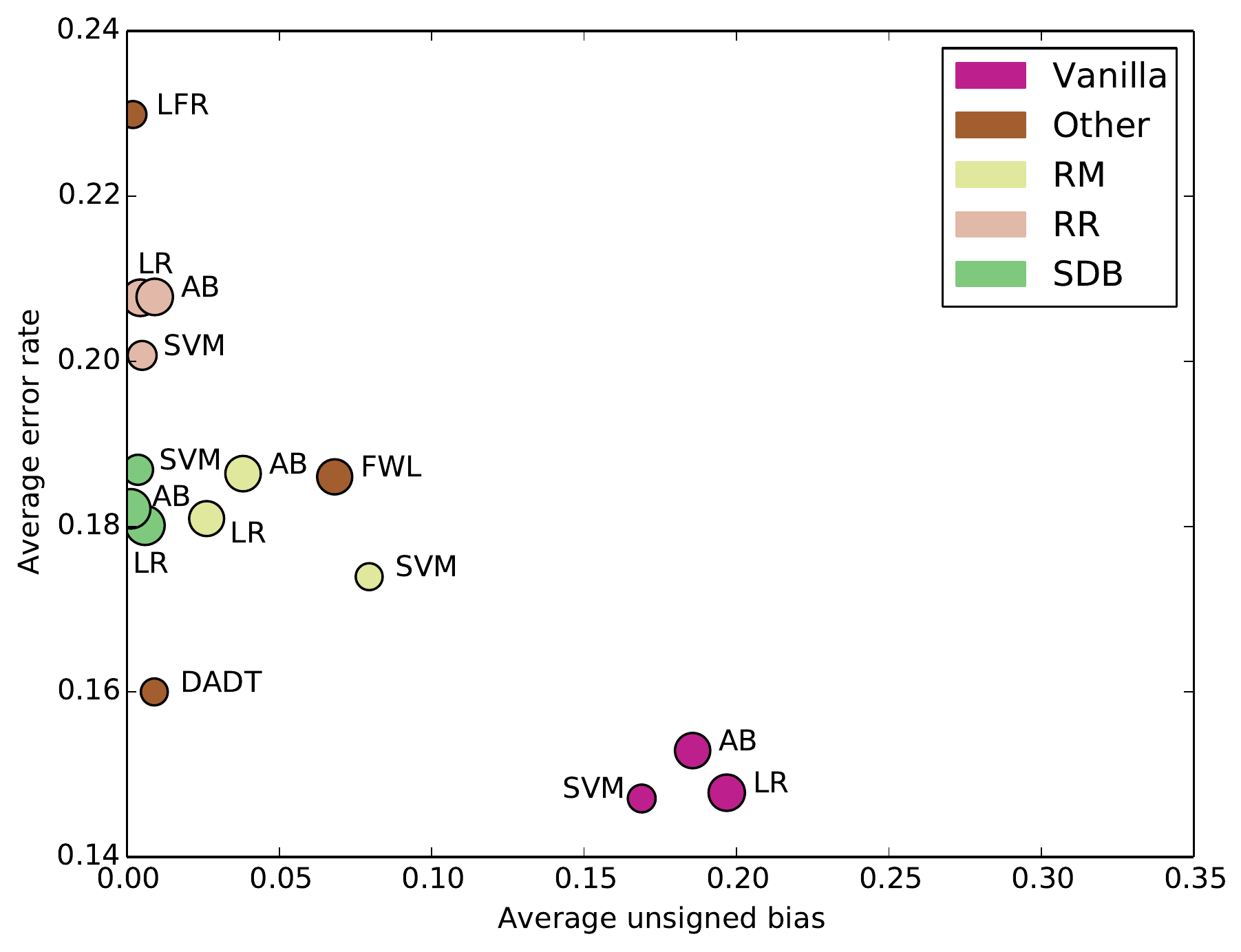}
\end{subfigure}
\hspace{1mm}
\begin{subfigure}{0.65\columnwidth}
\includegraphics[width=\columnwidth]{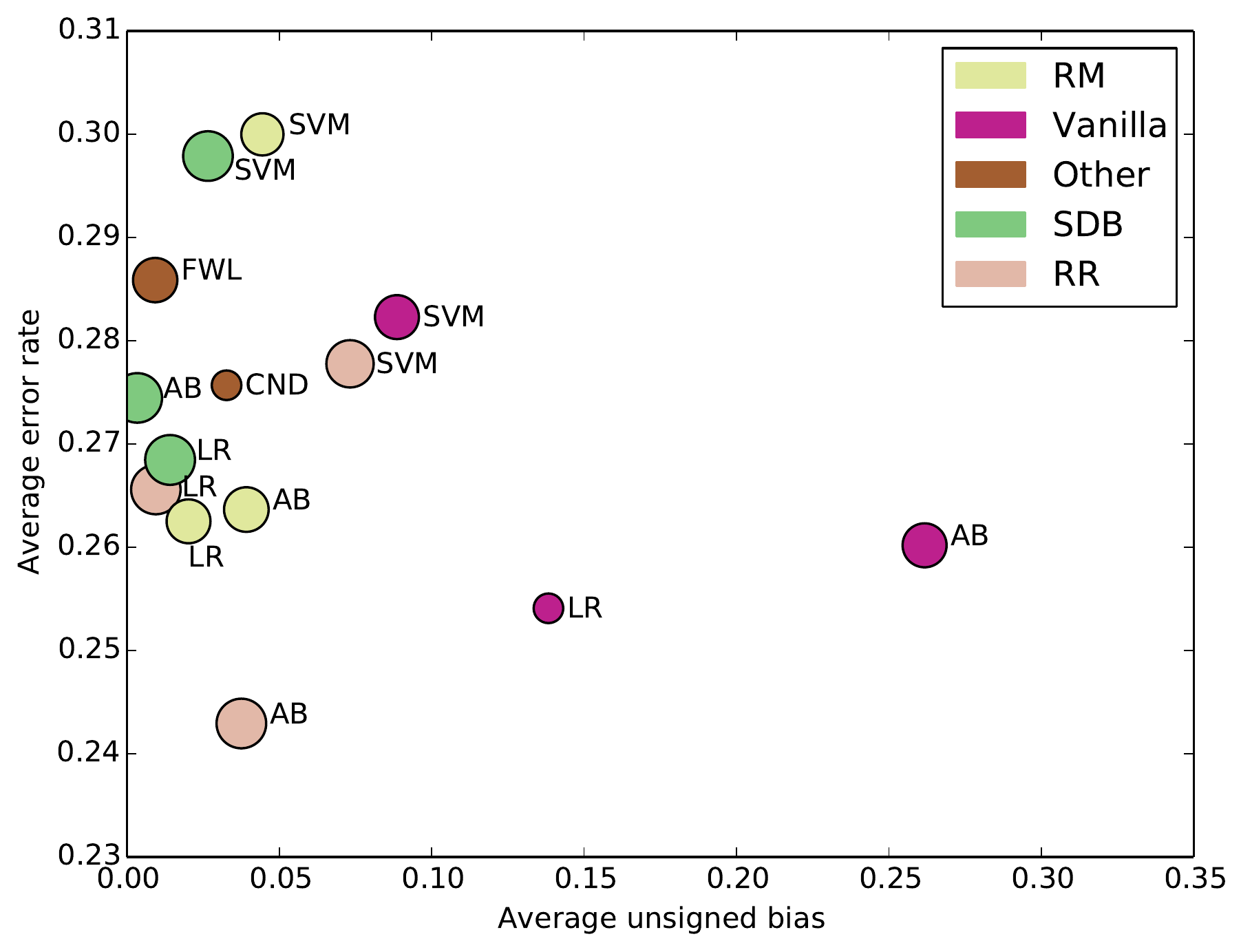}
\end{subfigure}
\hspace{1mm}
\begin{subfigure}{0.65\columnwidth}
\includegraphics[width=\columnwidth]{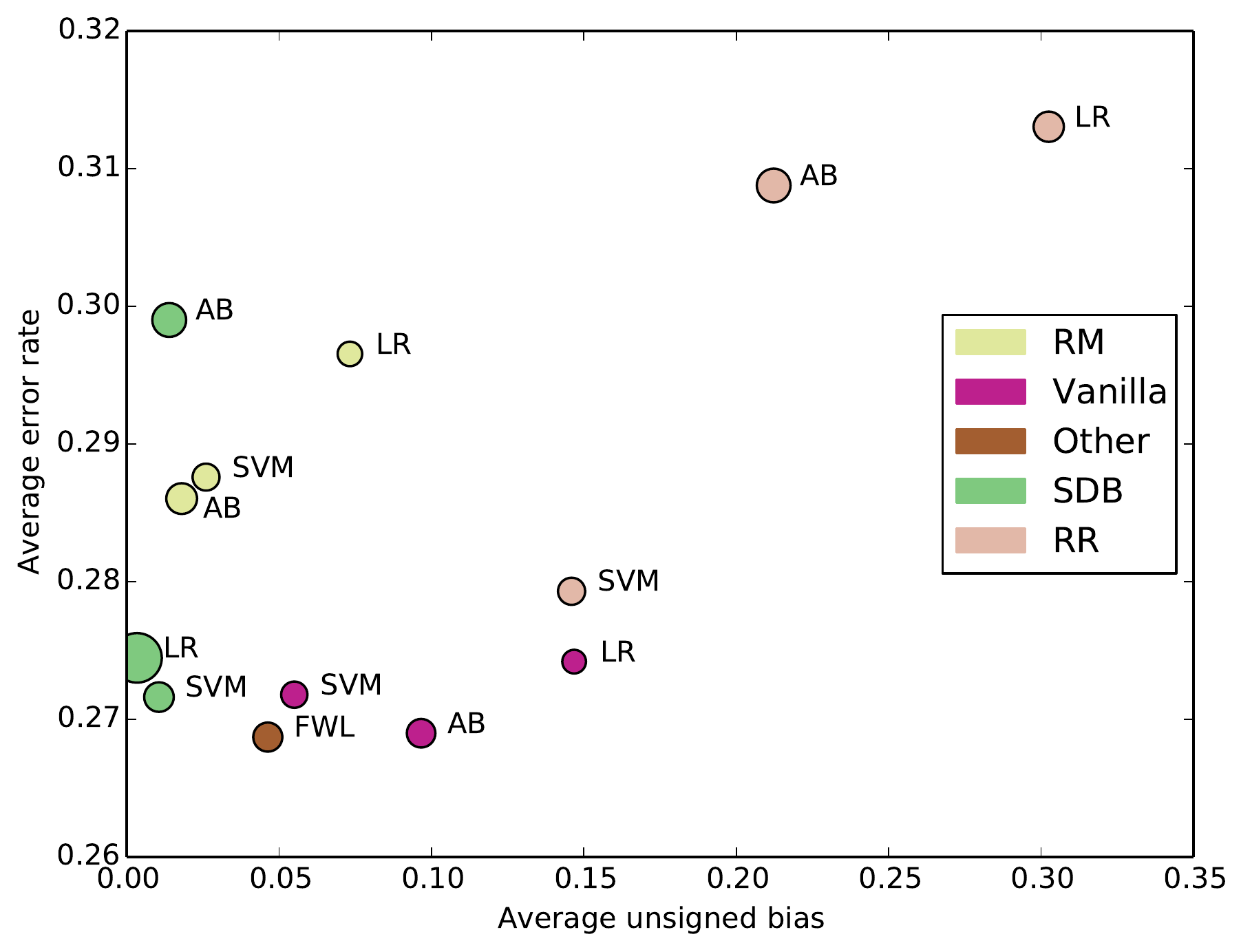}
\end{subfigure}
\caption{A summary of our experimental results for our three datasets, from
left to right: Census Income, German, Singles.
Labels show which learning algorithm is used and the colors give which method for achieving fairness was used.
The parameters of each algorithm
were chosen to minimize bias. The size of a point is proportional to the RRB of
the learner (only for those algorithms for which we have the RRB numbers),
where larger dots mean there is a larger probability of correcting
a label.} 
\label{fig:scatterplots} 
\end{figure*}

\begin{table}[h]
\centering
\begin{tabular}{| c | ccc |}
\hline
               Method    & Census & German & Singles \\ 
\hline 
               SVM       & 0.2702 & 0.6756 & 0.2424  \\ 
               SVM (RR)  & 0.2821 & 0.7827 & 0.2588  \\ 
               SVM (RM)  & 0.2545 & 0.6232 & 0.2552  \\ 
               SVM (SDB) & \textbf{0.3172} & \textbf{0.8619} & \textbf{0.3064}  \\ 
\hline
               LR        & 0.4647 & 0.3070 & 0.1971  \\ 
               LR (RR)   & 0.4696 & 0.8564 & 0.3213  \\ 
               LR (RM)   & 0.4282 & 0.6741 & 0.2117  \\ 
               LR (SDB)  & \textbf{0.5402} & \textbf{0.8687} & \textbf{0.8596}  \\ 
\hline
               AB        & 0.4372 & 0.6774 & 0.2864  \\ 
               AB (RR)   & 0.4661 & \textbf{0.8629} & 0.3996  \\ 
               AB (RM)   & 0.4410 & 0.6965 & 0.3325  \\ 
               AB (SDB)  & \textbf{0.5461} & 0.8596 & \textbf{0.4027}  \\ 
               AB (FWL)  & 0.5174 & 0.6879 & 0.2971  \\ 
\hline
\end{tabular}
\caption{The RRB numbers for each of our methods and baselines. In each column
and section the largest values are shown in bold, and they are almost always
SDB.}
\label{table:rrb}
\end{table}

In this section we state our experimental results. They are summarized in
Figure~\ref{fig:scatterplots} for the Census Income, German, and Singles
datasets, and the full set of numbers are in Tables~\ref{table:census_results},~\ref{table:german_results}, and~\ref{table:singles_results} respectively.  For comparison,
we also included the numbers for the Learning Fair Representations (LFR) method
of~\cite{ZemelWSPD13} for the Census Income dataset, for Classification with No
Discrimination (CND) method of~\cite{KamiranC09}, and for the Discrimination
Aware Decision Tree (DADT) technique of~\cite{KamiranCP10} (specifically we use
the numbers for the ``IGC+IGS\_Relab'' method). In~\cite{ZemelWSPD13} the
authors implemented three other learning algorithms, these are unregularized
logistic regression, Fair Naive-Bayes \cite{KamiranC09}, and Regularized
Logistic Regression \cite{KamashimaAS11}. These methods all had errors above
$20\%$ on the Census dataset and so we omit them for brevity.
In~\cite{KamiranCP10} the authors implemented variations on the decision tree
learning scheme, and the one we include has the highest accuracy, though they
are all closely comparable.  We reported all biases as unsigned. We were unable to access
implementations of the prior authors' algorithms, so we were not able to
reproduce their results or measure their algorithms with respect to RRB.

To investigate the trade-offs made by our SDB
method more closely,
Figures~\ref{fig:adult_tradeoffs},~\ref{fig:german_tradeoffs},
and~\ref{fig:singles_tradeoffs} show the rate at which error increases as bias
goes to zero.  In many cases, a substantial reduction in bias can be achieved before there is any significant drop-off in accuracy.

For the Census Income dataset, the three SDB techniques outperform the
baselines and outperform all the prior literature except for DADT. Both SDB
algorithms achieve statistical parity with about $18\%$ error. Moreover, these
two SDB algorithms have the highest RRB, while SVM appears to overfit the
random bias introduced by RRB more than the other algorithms. While DADT
appears to achieve lower label error and comparable bias, we note that the
standard deviation of the bias reported in~\cite{KamiranCP10} is 0.015 while for SDB (on the Census Income dataset) the standard deviations are at
least one order of magnitude smaller. 

The singles dataset shows a similar pattern, with SDB combined with logistic
regression outperforming all other baselines. Note that in the instances where
the baselines perform comparably to SDB, SDB tends to have a much larger
resilience to random bias.

The German dataset is more puzzling. While two of the SDB techniques
outperform the prior literature by a moderate margin, they do not 
outperform random relabeling or random massaging by a significant
margin (and these baselines already outperform CND). Another curious
observation is that label error stays
constant as the decision boundary is shifted, as Figure~\ref{fig:german_tradeoffs} shows.

Note again the difference in SVM kernels between the datasets. The Gaussian
kernel performs well for the Census Income and Singles dataset. However, in
the case of the German dataset, which is the smallest of the three, with the
Gaussian kernel every point becomes a support vector. This is not only a clear
sign of overfitting but it also makes SDB useless since the model gives the
same confidence for almost every data point.

These facts seem to be
evidence that the German dataset (which has only about a thousand records) is
too small to draw a significant conclusion. We nevertheless include it here for
completeness and to show comparison with the previous literature.

Fair weak learning (FWL) does empirically reduce bias but does not achieve
statistical parity in two of the three data sets.  FWL performs worse
on either label error or bias on each of the data sets and the trade-off
between label error and bias cannot easily be controlled. It also does not seem
easy to control this trade-off using either random massaging and random
relabeling.
 
One notable advantage of SDB is that the trade-off between label error and bias
can be controlled \emph{after} training.  To decide how much bias and error we
want to allow, we do not have to pick a hyper-parameter before training the
algorithm, unlike for most other fair learning methods. This means that the
computational cost of choosing the best point on the trade-off curve is very
low, and the trade-off is transparent. 

The results also highlight the usefulness of RRB as a measure of fairness. The
RRB values across all datasets and algorithms we studied are in
Table~\ref{table:rrb}.  In cases where random relabeling or random massaging
performs comparably to SDB, the RRB measure is able to distinguish them, giving
a lower score to the less reasonable baselines and a higher score to SDB.
This suggests that
the performance of fair learning algorithms should not be evaluated solely by
their accuracy and bias.

\section{Significance and Impact}\label{sec:discussion}
In this paper, we introduced a general method for balancing discrimination and
label error. This method, which we call shifted decision boundary (SDB), is
applicable to any learning algorithm which has an efficiently computable
measure of confidence. We studied three such algorithms -- AdaBoost, support
vector machines, and linear regression -- compared our methods to other methods
proposed in the earlier literature and our own baselines, and empirically
evaluated our methods' performances in terms of their resilience to random bias. 

Our method, in addition to outperforming much of the previous literature, has
several other desirable properties. Unlike most other fair learning algorithms,
SDB applied to AdaBoost has theoretical bounds on generalization error.
Also, since the margin shift can be specified after the
original learner has been trained on the data, a practitioner can easily
evaluate the trade-off between error and bias and choose the most desirable
point on the trade-off curve. This makes SDB a fast and transparent way to study
the fairness properties of an algorithm.

Our resilience to random bias (RRB) measure is a novel approach to evaluate the
fairness of a learning algorithm. Although i.i.d.~random bias is a simplified
model of real-world discrimination, we posit that any algorithm which can be
considered fair must be fair with respect to RRB. Moreover, RRB generalizes to
an arbitrary distribution over the input data, and one could adapt it to
well-studied models of bias in social science.

\section*{Acknowledments}
We would like to thank Lev Reyzin for helpful discussions.

\begin{figure*}[t]
\centering
\begin{subfigure}{.7\columnwidth}
\includegraphics[width=\columnwidth]{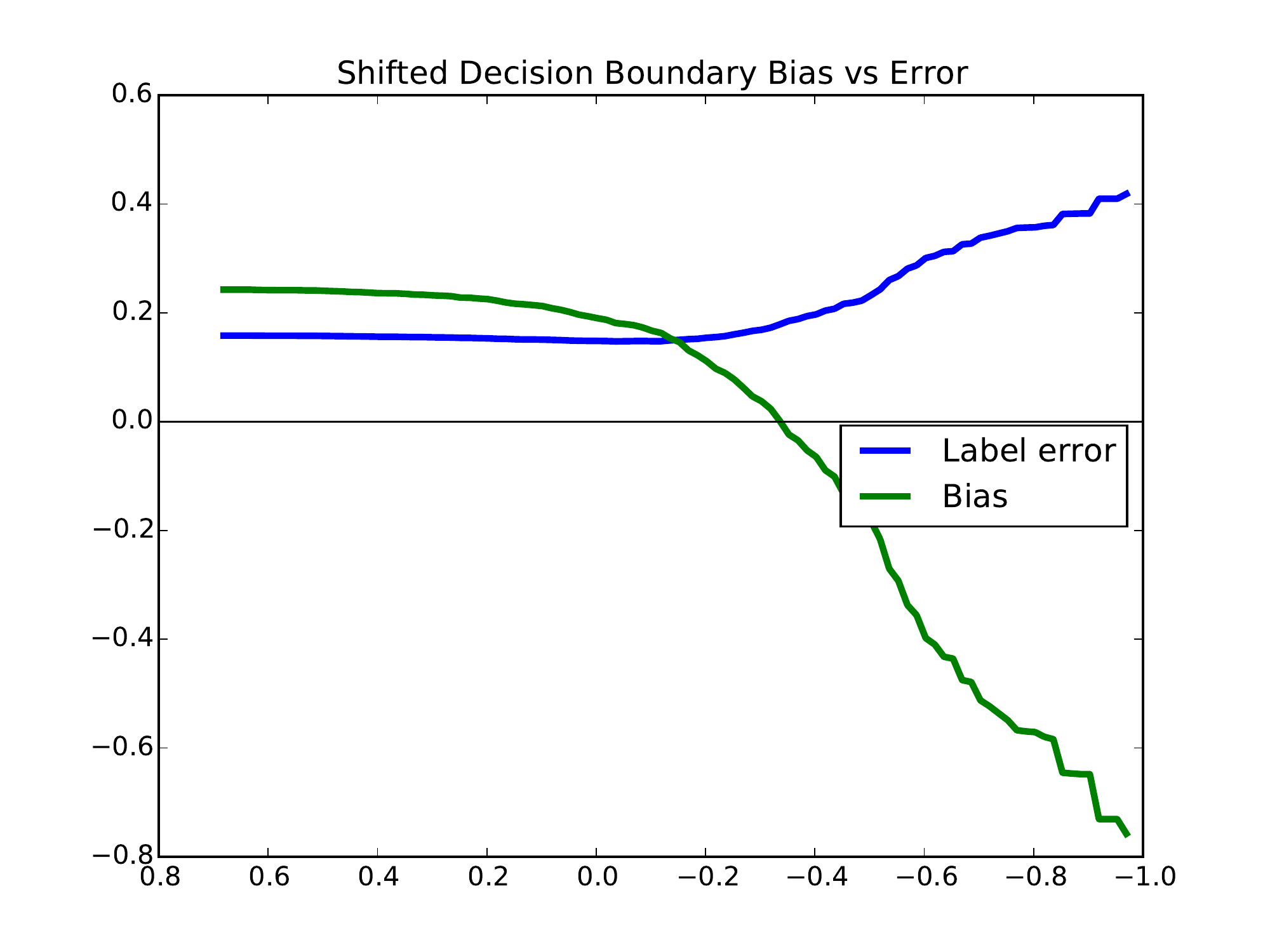}%
\caption{Boosting}%
\label{fig:adult_boosting_tradeoff}%
\end{subfigure}
\begin{subfigure}{.7\columnwidth}
\includegraphics[width=\columnwidth]{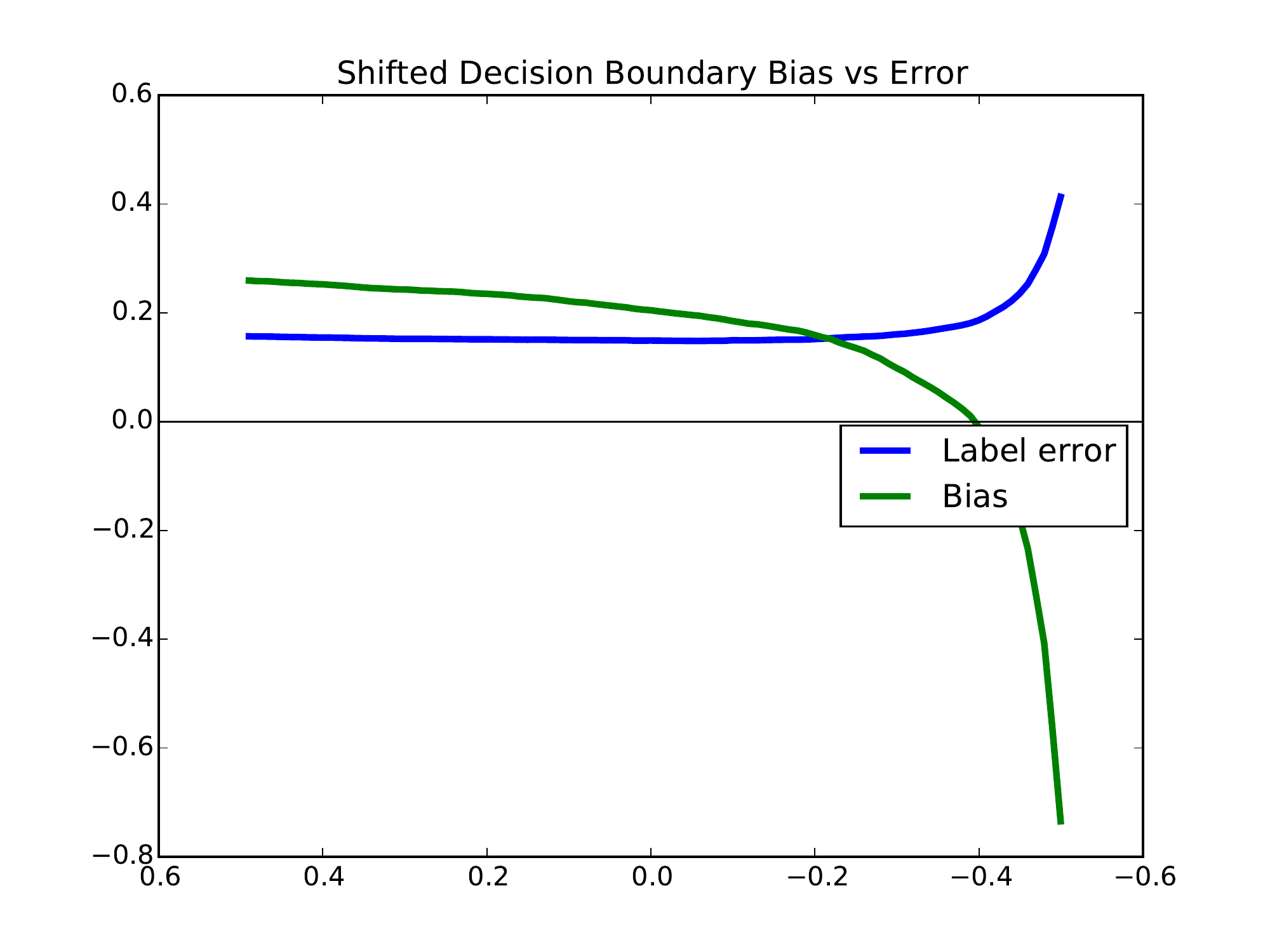}%
\caption{Logistic Regression}%
\label{fig:adult_lr_tradeoff}%
\end{subfigure}
\begin{subfigure}{.7\columnwidth}
\includegraphics[width=\columnwidth]{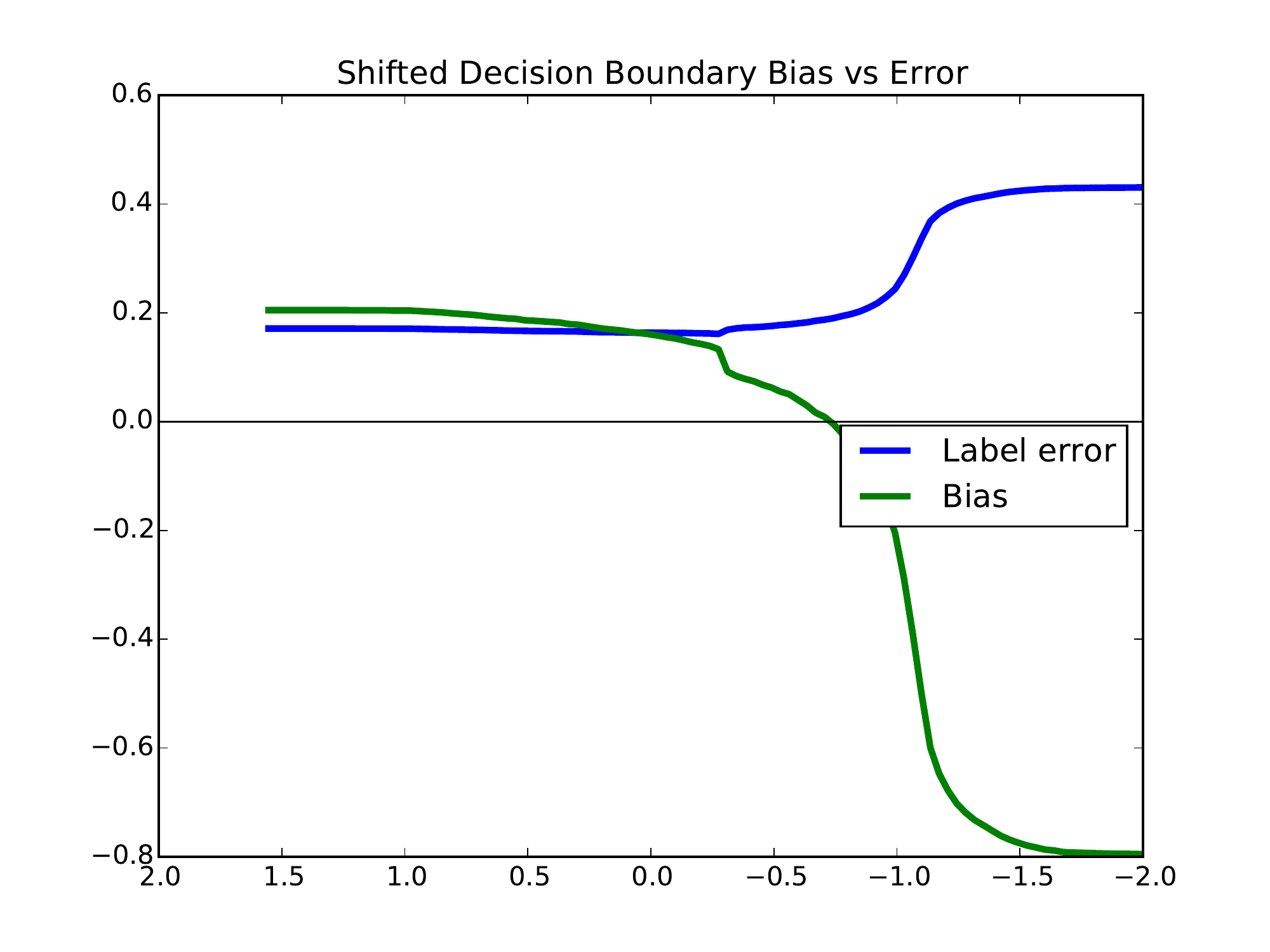}%
\caption{SVM}%
\label{fig:adult_svm_tradeoff}%
\end{subfigure}%
\caption{Trade-off between (signed) bias and error for SDB on the Census Income data. The horizontal axis is the threshold used for SDB.}
\label{fig:adult_tradeoffs}
\end{figure*}

\begin{figure*}[t]
\centering
\begin{subfigure}{.7\columnwidth}
\includegraphics[width=\columnwidth]{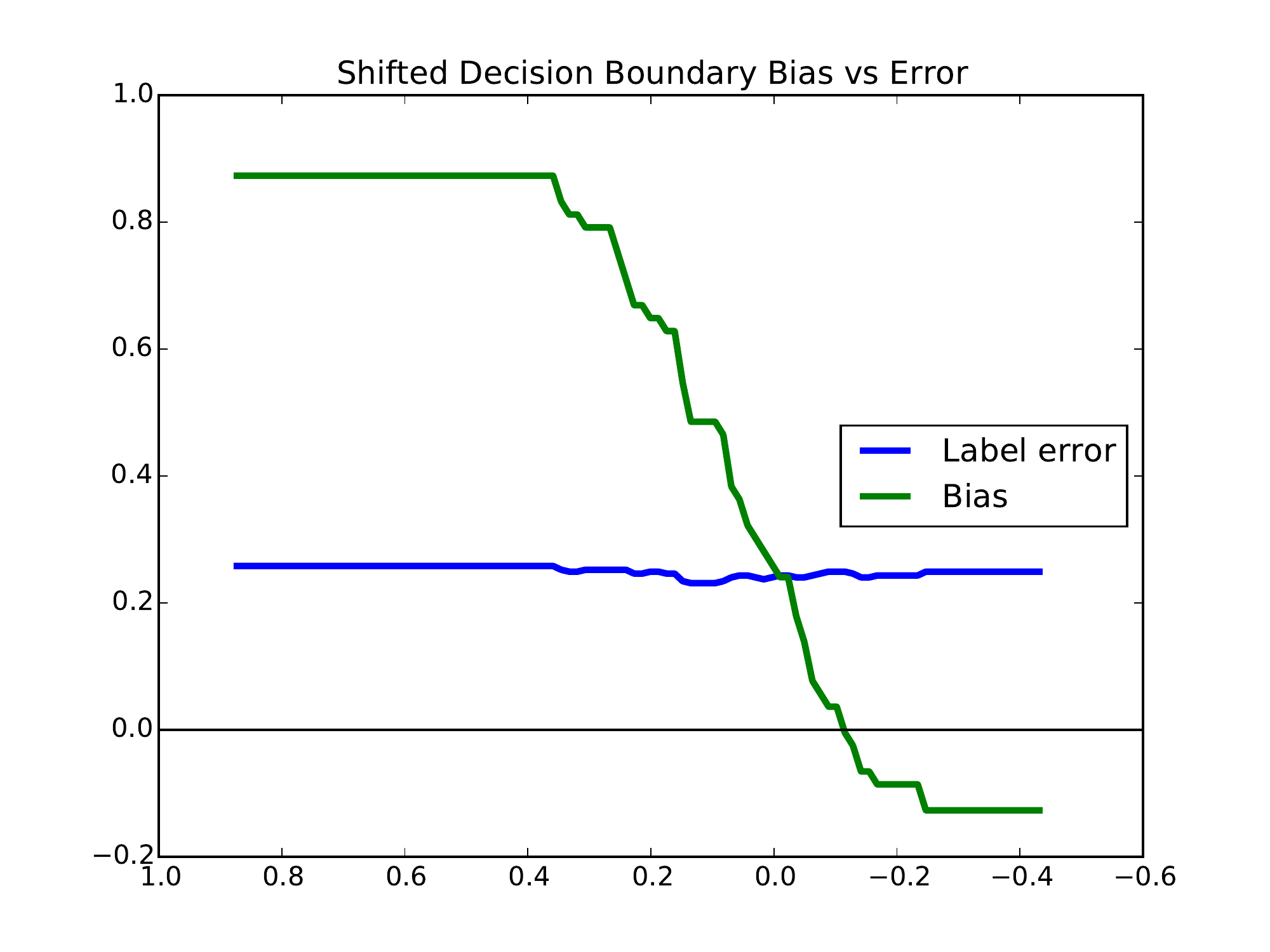}%
\caption{Boosting}%
\label{fig:german_boosting_tradeoff}%
\end{subfigure}
\begin{subfigure}{.7\columnwidth}
\includegraphics[width=\columnwidth]{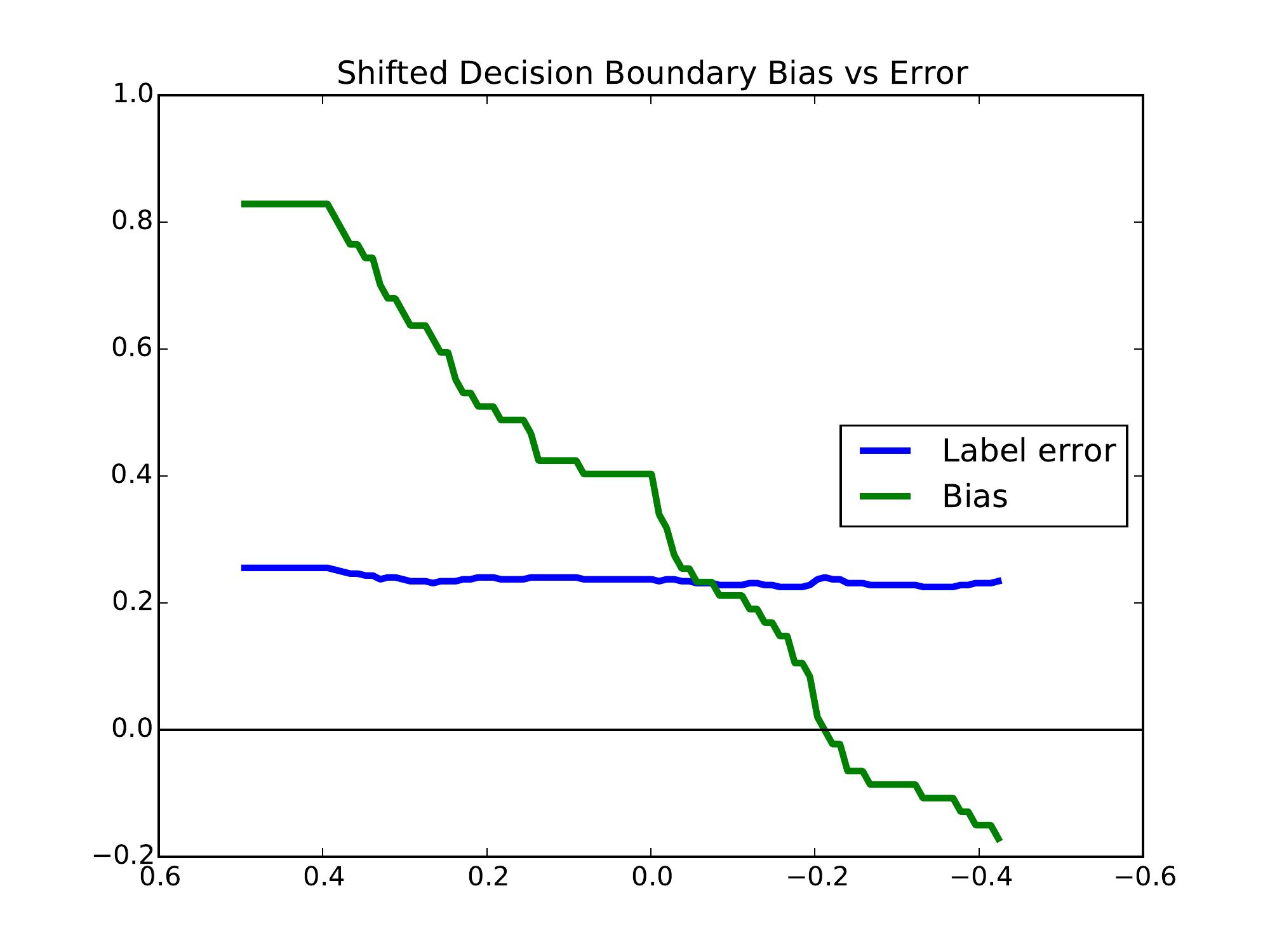}%
\caption{Logistic Regression}%
\label{fig:german_lr_tradeoff}%
\end{subfigure}
\begin{subfigure}{.7\columnwidth}
\includegraphics[width=\columnwidth]{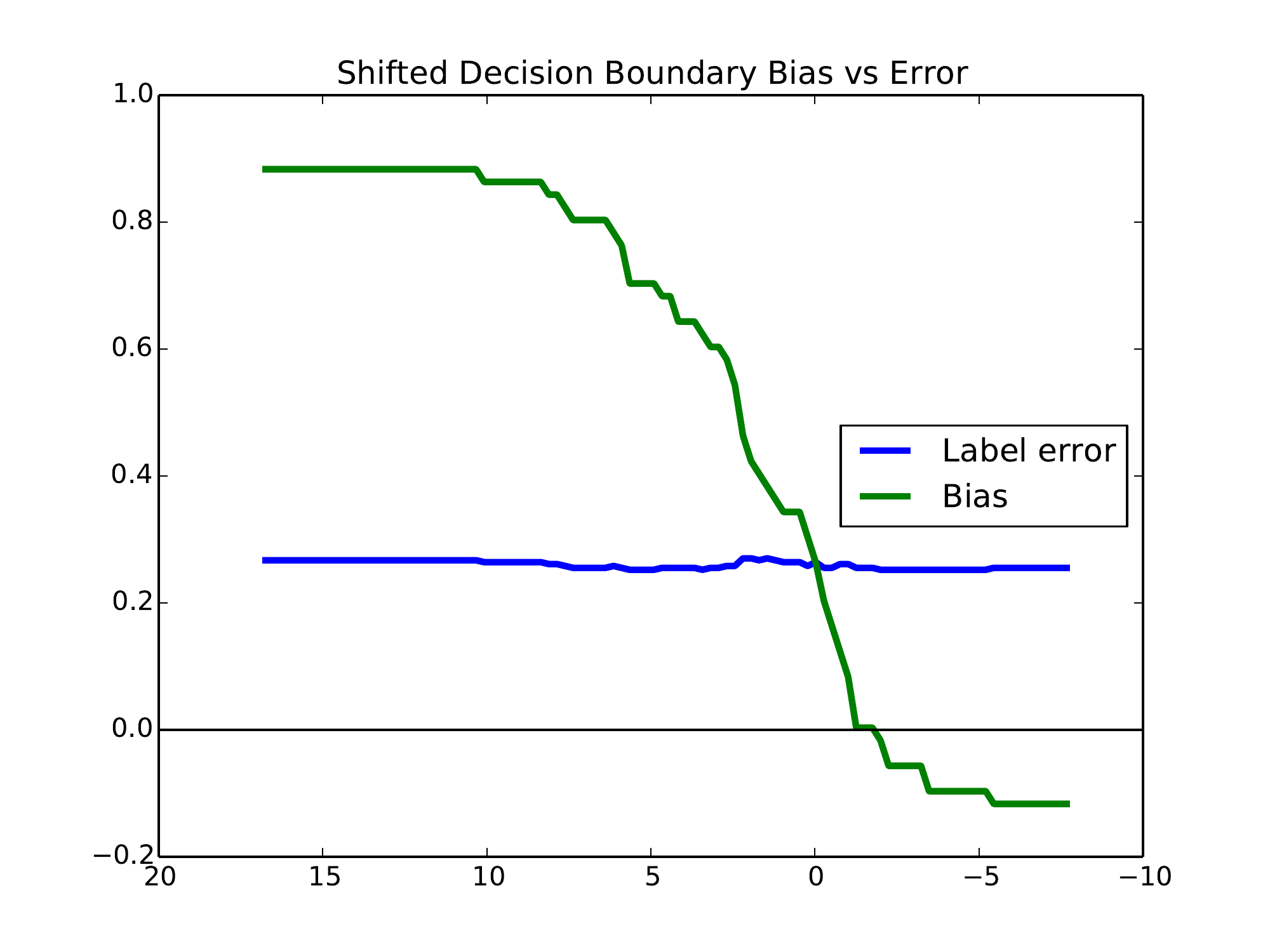}%
\caption{SVM}%
\label{fig:german_svm_tradeoff}%
\end{subfigure}%
\caption{Trade-off between (signed) bias and error for SDB on the German data. The horizontal axis is the threshold used for SDB.}
\label{fig:german_tradeoffs}
\end{figure*}

\begin{figure*}[t]
\centering
\begin{subfigure}{.7\columnwidth}
\includegraphics[width=\columnwidth]{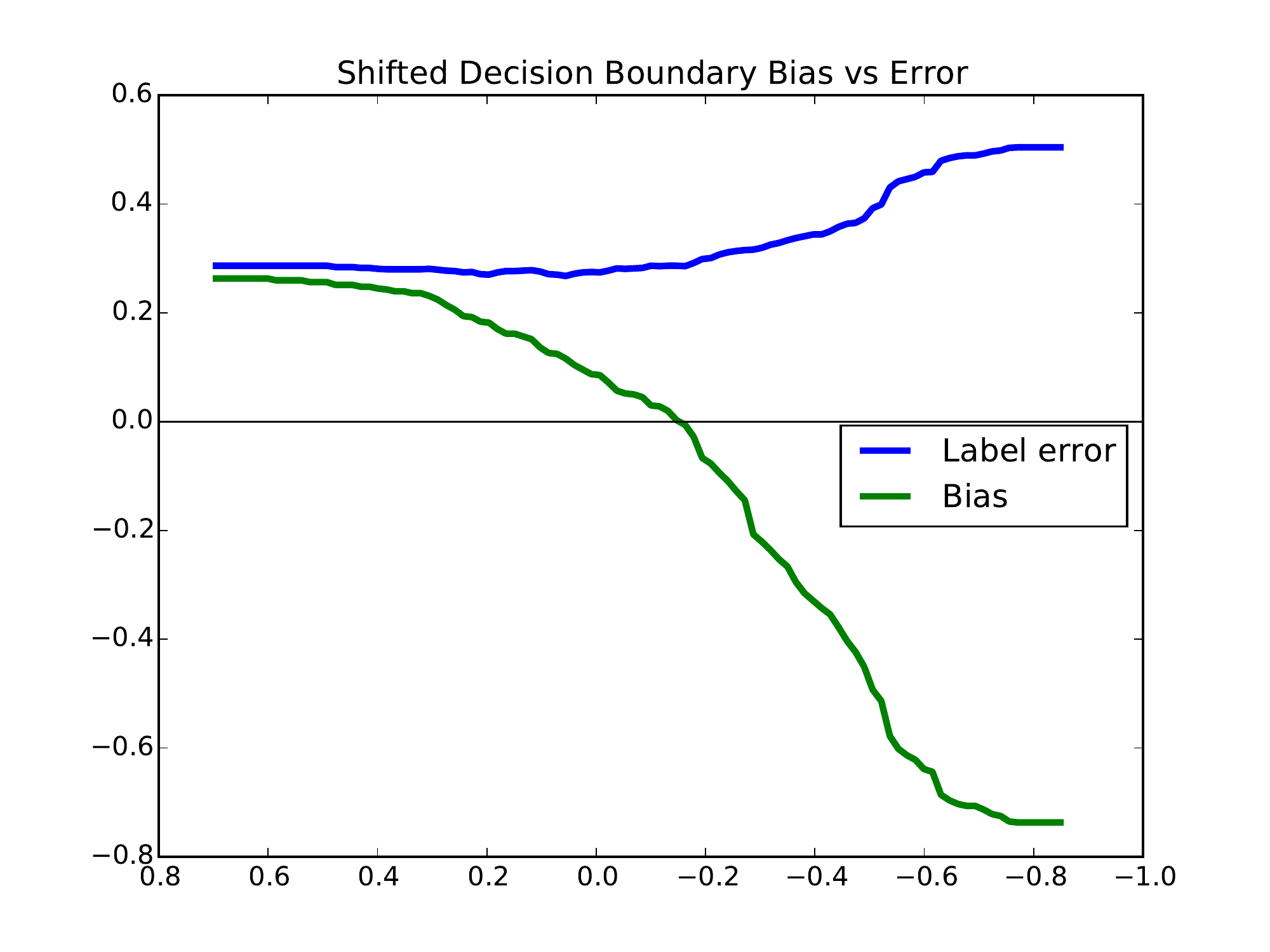}%
\caption{Boosting}%
\label{fig:singles_boosting_tradeoff}%
\end{subfigure}
\begin{subfigure}{.7\columnwidth}
\includegraphics[width=\columnwidth]{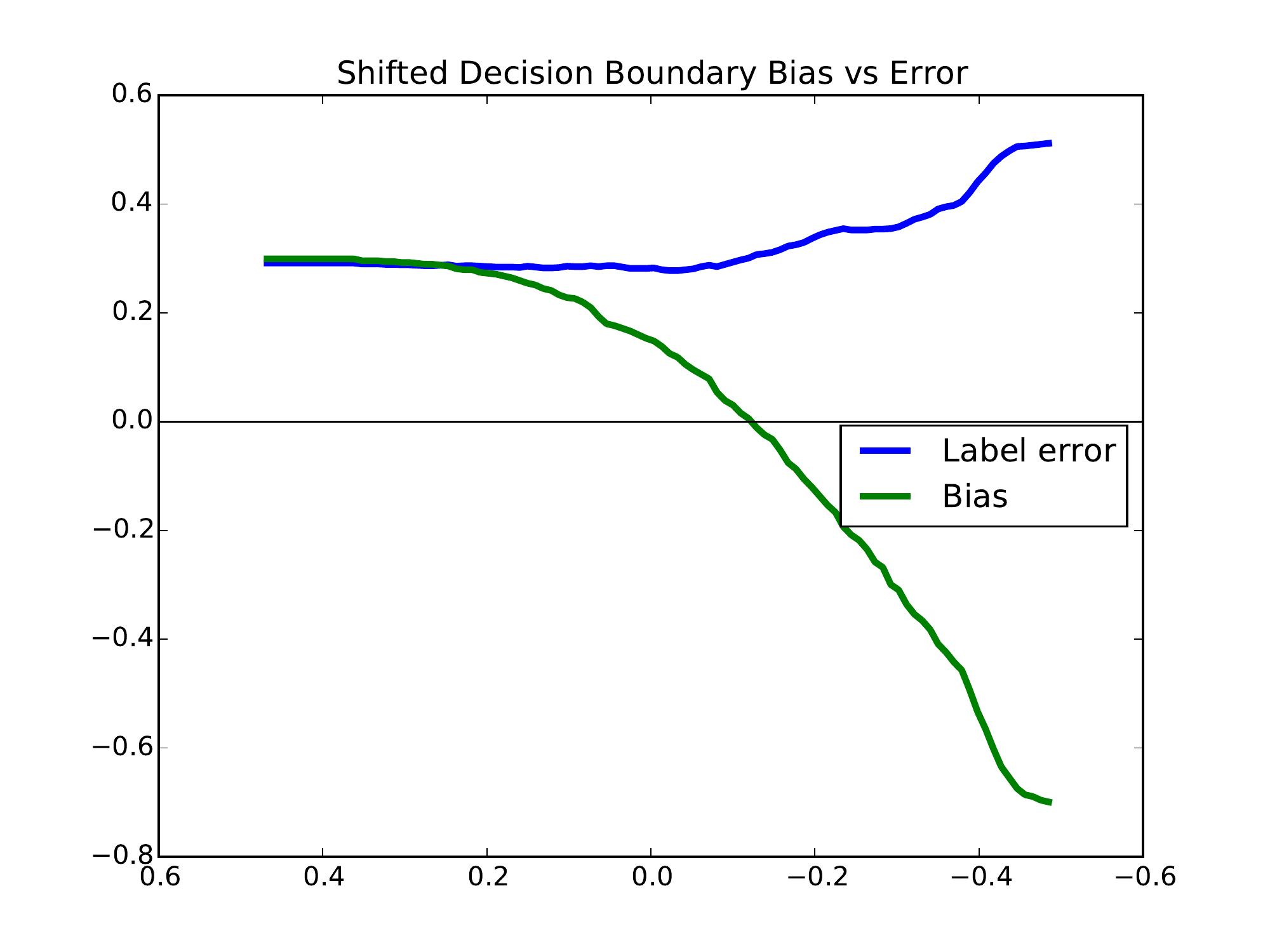}%
\caption{Logistic Regression}%
\label{fig:singles_lr_tradeoff}%
\end{subfigure}
\begin{subfigure}{.7\columnwidth}
\includegraphics[width=\columnwidth]{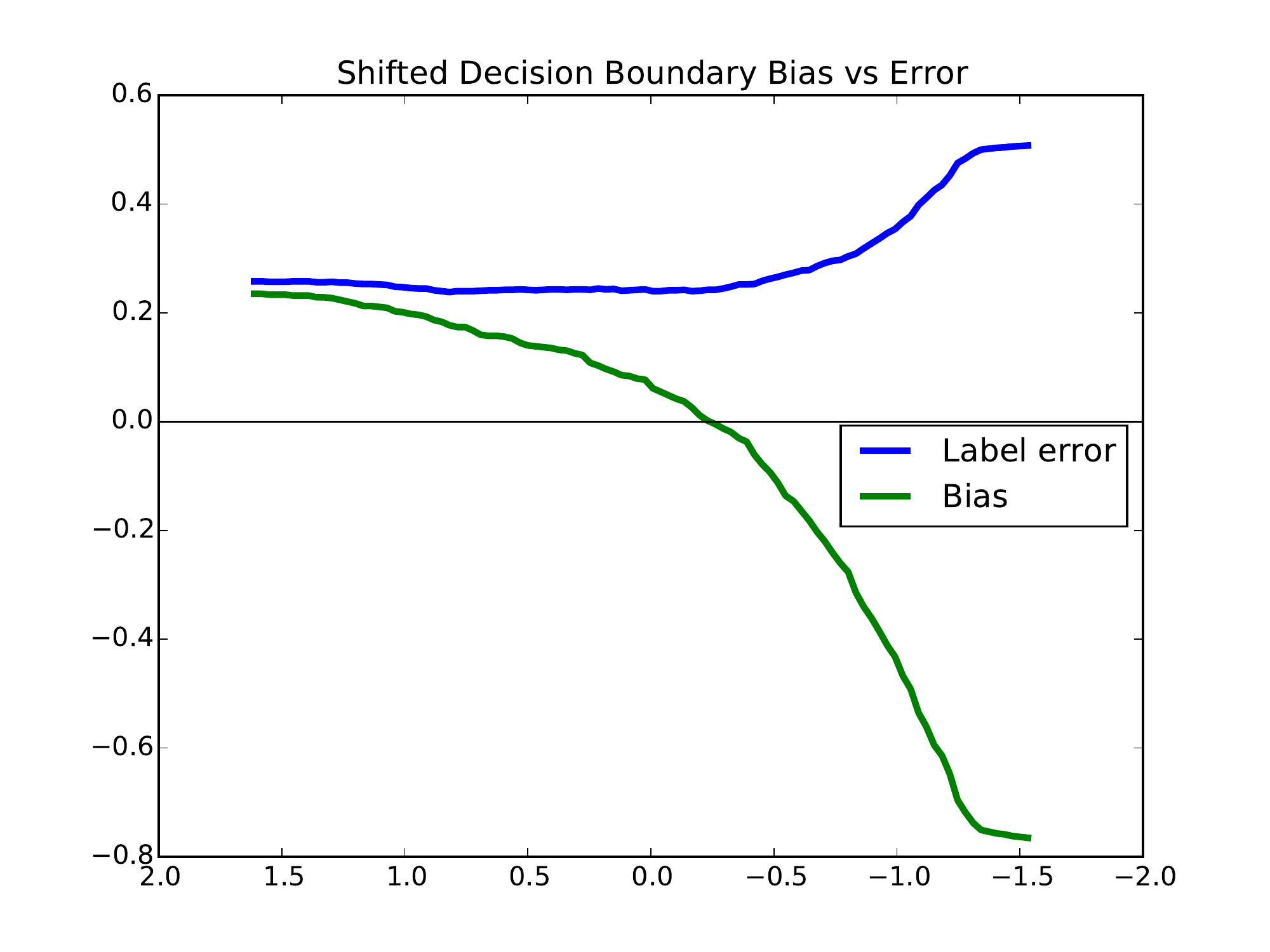}%
\caption{SVM}%
\label{fig:singles_svm_tradeoff}%
\end{subfigure}%
\caption{Trade-off between (signed) bias and error for SDB on the Singles data. The horizontal axis is the threshold used for SDB.}
\label{fig:singles_tradeoffs}
\end{figure*}

\bibliographystyle{plain}
\bibliography{main}

\clearpage
\appendix

\begin{table*}[t]
\centering
\begin{tabular}{| c | ccccc |}
\hline
               & SVM & SVM (RR) & SVM (SDB) & SVM (RM) & LFR \cite{ZemelWSPD13} \\
\hline
label error    & 0.1471 (5.7e-17) & 0.2007 (0.002) & 0.1869 (0.004) & 0.1740 (0.003) & 0.2299 \\
bias           & 0.1689 (5.7e-17) & 0.0050 (0.003) & 0.0036 (0.009) & 0.0795 (0.010) & 0.0020 \\
RRB            & 0.2702 (0.014) & 0.2926 (0.004) & 0.3172 (0.025) & 0.2545 (0.007) & n/a \\
\hline
               & LR & LR (RR) & LR (SDB) & LR (RM) & DADT \cite{KamiranCP10} \\
\hline
label error    & 0.1478 (4.8e-04) & 0.2077 (0.004) & 0.1802 (0.002) & 0.1810 (0.003) & 0.1600 \\
bias           & 0.1968 (0.003) & 0.0044 (0.006) & 0.0060 (0.011) & 0.0262 (0.008) & 0.0090 (0.015) \\
RRB            & 0.4647 (0.013) & 0.4696 (0.009) & 0.5402 (0.011) & 0.4282 (0.019) & n/a \\
\hline
               & AdaBoost & AB (RR) & AB (SDB) & AB (RM) & AB (FWL)  \\
\hline
label error    & 0.1529 (0.002) & 0.2078 (0.004) & 0.1822 (0.005) & 0.1864 (0.004) & 0.1860 (0.004) \\
bias           & 0.1856 (0.012) & 0.0091 (0.006) & 0.0013 (0.007) & 0.0381 (0.013) & 0.0682 (0.004)  \\
RRB            & 0.4372 (0.032) & 0.4661 (0.019) & 0.5461 (0.015) & 0.4410 (0.013) & 0.4321 (0.016)  \\
\hline
\hline
\end{tabular}
\caption{A summary of our experimental results for the Census Income data for
relabeling, massaging, and the fair weak learner. The threshold for SDB was
chosen to achieve perfect statistical parity on the training data. Standard
deviations are reported in parentheses when known.}
\label{table:census_results}
\end{table*}

\begin{table*}[h]
\centering
\begin{tabular}{| c | ccccc |}
\hline
               & SVM & SVM (RR) & SVM (SDB) & SVM (RM) & CND \cite{KamiranC09} \\
\hline
label error    & 0.2823 (0) & 0.2778 (0.025) & 0.2979 (0.022) & 0.3000 (0.017) & 0.2757  \\
bias           & 0.0886 (4.2e-17) & 0.0732 (0.066) & 0.0266 (0.085) & 0.0445 (0.028) & 0.0327   \\
RRB            & 0.6756 (0.081) & 0.7827 (0.054) & 0.8619 (0.041) & 0.6232 (0.070) & n/a  \\
\hline
               & LR & LR (RR) & LR (SDB) & LR (RM) &  \\
\hline
label error    & 0.2541 (0.005) & 0.2656 (0.020) & 0.2685 (0.021) & 0.2625 (0.011) &\\
bias           & 0.1383 (0.014) & 0.0095 (0.064) & 0.0142 (0.219) & 0.0202 (0.566) &\\
RRB            & 0.3070 (0.067) & 0.8564 (0.045) & 0.8687 (0.042) & 0.6741 (0.045) & \\
\hline
               & AdaBoost & AB (RR)  & AB (SDB)  & AB (RM)   & AB (FWL)  \\
\hline
label error    & 0.2602 (0.009) & 0.2429 (0.010) & 0.2745 (0.010) & 0.2637 (0.019) & 0.2859 (0.016)\\
bias           & 0.2617 (0.272) & 0.0376 (0.044) & 0.0034 (0.064) & 0.0391 (0.023) & 0.0093 (0.035)\\
RRB            & 0.6774 (0.219) & 0.8629 (0.051) & 0.8596 (0.067) & 0.6965 (0.037) & 0.6879 (0.042)\\
\hline
\hline
\end{tabular}
\caption{A summary of our experimental results for the German data for
relabeling, massaging, and the fair weak learner. The threshold for SDB was
chosen to achieve perfect statistical parity on the training data. On this
dataset SVM was run with a linear kernel. Standard deviations are reported in
parentheses when known. }
\label{table:german_results}
\end{table*}

\begin{table*}[h]
\centering
\begin{tabular}{| c | ccccc |}
\hline
               & SVM & SVM (RR) & SVM (SDB) & SVM (RM) & \\
\hline
label error    & 0.2718 (5.7e-17)& 0.2793 (0.009) & 0.2716 (0.013) & 0.2876 (0.015)& \\
bias           & 0.0550 (1.4e-17)& 0.1460 (0.017) & 0.0106 (0.035) & 0.0260 (0.047)& \\
RRB            & 0.2424 (0.045)& 0.2588 (0.009) & 0.3064 (0.042) & 0.2552 (0.032)&\\
\hline
               & LR & LR (RR)  & LR (SDB)  & LR (RM)  &  \\
\hline
label error    & 0.2742 (1.14e-16)& 0.3130 (0.011) & 0.2745 (0.010) & 0.2966 (0.008)&\\
bias           & 0.1468 (9.99e-18)& 0.3025 (0.040) & 0.0034 (0.640) & 0.0732 (0.024)&\\
RRB            & 0.1971 (0.036)& 0.3213 (0.035) & 0.8596 (0.067) & 0.2117 (0.036)& \\
\hline
               & AdaBoost & AB (RR) & AB (SDB) & AB (RM) & AB (FWL)   \\
\hline
label error    & 0.2690 (0.004)& 0.3088 (0.009) & 0.2990 (0.008) & 0.2860 (0.019) & 0.2687 (0.008)\\
bias           & 0.0966 (0.020)& 0.2123 (0.013) & 0.0140 (0.017) & 0.0180 (0.037) & 0.0463 (0.016)\\
RRB            & 0.2864 (0.057)& 0.3996 (0.105) & 0.4027 (0.061) & 0.3325 (0.060) & 0.2971 (0.028)\\
\hline
\hline
\end{tabular}
\caption{A summary of our experimental results for the Singles data for
relabeling, massaging, and the fair weak learner. The threshold for SDB was
chosen to achieve perfect statistical parity on the training data. Standard
deviations are reported in parentheses when known. }
\label{table:singles_results}
\end{table*}

\end{document}